\documentclass[runningheads]{llncs}

\usepackage{accv}

\usepackage{accvabbrv}

\usepackage{graphicx}
\usepackage{booktabs}
\usepackage{algorithm}
\usepackage{algorithmicx}
\usepackage{algpseudocode}
\usepackage{multicol}
\usepackage{pifont}
\newcommand{\cmark}{\ding{51}}
\newcommand{\xmark}{\ding{55}}
\usepackage{multirow}
\usepackage[accsupp]{axessibility}  

\usepackage[pagebackref,breaklinks,colorlinks,citecolor=accvblue]{hyperref}

\usepackage{orcidlink}

\begin{document}

\title{TaDA: Calibrated Probe Gating for \\ Task-Domain LoRA Merging} 

\titlerunning{Task-Domain LoRA Merging}

\author{Huy Quoc To\inst{1}\orcidlink{0000-0002-0936-245X} \and
Fuyi Li\inst{2}\orcidlink{0000-0001-5216-3213} \and
Guangyan Huang\inst{1}\orcidlink{0000-0002-1821-8644} \and
Ming Liu\inst{1}\orcidlink{0000-0002-2160-6111}}

\authorrunning{Huy To et al.}

\institute{Deakin University, Melbourne, VIC, Australia \and
Adelaide University, Adelaide, SA, Australia \\
\email{\{q.to, guangyan.huang, m.liu\}@deakin.edu.au, fuyi.li@adelaide.edu.au}}

\maketitle

\begin{abstract}
Combining a task LoRA adapter with a domain LoRA adapter into a
single unified model is a practical yet largely unexplored challenge.
Existing methods treat both adapters as symmetric peers, applying
uniform weights across all layers.
We argue that task and domain adapters exhibit a consistent
depth-dependent asymmetry across transformer architectures.
Domain dominance increases with layer depth, while shallower layers
retain stronger task-relevant signals.
Motivated by this observation, we propose \textbf{TaDA}
(\textbf{Ta}sk-\textbf{D}omain LoR\textbf{A} Merging), a
training-free algorithm that exploits this structure through
calibrated probe-guided per-layer gating and per-component
subspace-aware merging.
The gating assigns individual weights per layer and projection type
using a probe signal proved invariant to adapter weight magnitude.
The merging discards conflicting singular directions before
combining the remaining components.
\textbf{TaDA} produces a standard rank-$r$ LoRA adapter with zero inference overhead.
On six scientific QA benchmarks with Llama-2-7B, TaDA achieves an average
accuracy of 0.452, outperforming DARE-TIES by +3.6 percentage points and
obtaining the best result on all six benchmarks. On six image classification
benchmarks with ViT-L/16, TaDA reaches 85.9\% average accuracy, improving
over the strongest merging baseline while leading in three of the six individual
benchmarks.

\end{abstract}




\section{Introduction}
\label{sec:intro}

The ecosystem of Low-Rank Adaptation (LoRA) adapters~\cite{hu2022lora}
has grown rapidly: model hubs now host thousands of adapters, each
specializing in either a \emph{task} (\eg instruction following) or a
\emph{domain} (\eg biomedical text, satellite imagery).
Merging a task adapter and a domain adapter into a single unified adapter capable of answering domain-specific queries without retraining is a practical challenge. However, it remains largely unexplored. Existing merging methods~\cite{wortsman2022modelsoupsaveragingweights,
yadav2023tiesmerging,yu2024dare} combine adapter weights
uniformly at every layer, treating task and domain adapters as symmetric
peers. However, we discover an interesting asymmetry. The task and domain adapters activate differently across transformer depth, and exploiting this structure leads to substantially better merges.

\paragraph{A diagnostic finding.}
By passing $N{=}32$ domain-specific inputs through both adapters
and measuring their activation norms at each layer, we find a
consistent structural pattern across both language (Llama-2-7b)
and vision (ViT-L/16) transformers. The domain dominance increases with depth. Shallower layers retain more task-relevant signal, while deeper layers are more strongly domain-dominant.
However, measuring raw activation norms directly is unreliable. The norms grow with training duration rather than domain relevance.
We address this with \textbf{calibrated probe scoring}. The domain relevance at each layer is computed as the ratio of activations on domain-specific inputs to activations on general inputs. We prove that this ratio is invariant to adapter weight magnitude.
Uniform merging ignores this structure entirely. Assigning the same weight at every layer gives the domain adapter too much influence in shallower layers, degrading task format, and too little in deeper layers, diluting domain knowledge.

In this paper, we propose \textbf{TaDA} (\textbf{Ta}sk-\textbf{D}omain LoR\textbf{A}
Merging), a training-free merging algorithm with two novelties.
First, \textbf{calibrated probe-guided gating} computes a per-layer,
per-module task weight using the calibrated score, with a
module-type-aware threshold $\tau_m$ that keeps output projections
task-dominant to preserve answer format.
Second, \textbf{per-component subspace-aware merging} decomposes both
adapters via SVD, discards singular components whose directions conflict
in both input and output feature spaces, and merges the remainder with
individual per-component weights. The merged adapter is a standard rank-$r$ LoRA with
zero inference overhead.

On six scientific text QA benchmarks using the Llama-2-7b backbone, \textbf{TaDA} achieves an average accuracy of $0.452$, outperforming DARE-TIES by $+3.6$pp. It achieves the best results on all six benchmarks, including MMLU-CS. On six image classification benchmarks using ViT-L/16, TaDA achieves an average accuracy of $85.9\%$, outperforming all baselines including DARE-TIES ($85.6\%$), while leading on three of the six individual benchmarks.

\paragraph{Our contributions are:}
\begin{itemize}
  \item We propose a novel merging algorithm for LoRA-based adapters. TaDA is a training-free merging algorithm that exploits the depth-dependent activation asymmetry between task and domain adapters via calibrated probe-guided per-layer gating and per-component subspace-aware filtering. Across both language and vision settings, TaDA achieves the best average performance among merging methods while preserving the standard rank-$r$ LoRA form and adding no inference-time overhead.

  \item We provide a systematic task $\times$ domain merging benchmark spanning
    six scientific text QA benchmarks using Llama-2-7b and six image
    classification benchmarks using ViT-L/16. Our results demonstrate that
 existing methods designed for task-to-task merging are
    inadequate for this asymmetric setting.
\end{itemize}


\section{Related Work}
\label{sec:related}

LoRA~\cite{hu2022lora} injects trainable low-rank matrices into
frozen model weights, reducing trainable parameters by orders of
magnitude. Subsequent work has extended this to vision
transformers~\cite{dosovitskiy2021an}, video
recognition~\cite{yang2023aim}, and medical
imaging~\cite{chen2024medblip}. Our work operates on trained LoRA
adapters rather than training them, making \textbf{TaDA} orthogonal to
improvements in LoRA training.

Model soups~\cite{wortsman2022modelsoupsaveragingweights} show that
averaging fine-tuned weights can improve accuracy and robustness.
Task Arithmetic~\cite{ilharco2023editing} proposes adding and
subtracting task vectors in weight space.
TIES~\cite{yadav2023tiesmerging} addresses parameter interference by
retaining only the top-magnitude values and resolving sign conflicts.
DARE~\cite{yu2024dare} randomly drops adapter weights before
rescaling, reducing interference stochastically.
KnOTS~\cite{stoica2025knots} uses SVD to align LoRA adapter
subspaces before merging, improving upon naive averaging for
task-to-task merging.

\textbf{TaDA} differs from all of the above in three ways, summarized in
Table~\ref{tab:related}:
(i)~it targets the \emph{asymmetric} task $\times$ domain setting,
where the two adapters encode qualitatively different knowledge;
(ii)~it assigns \emph{probe-guided per-layer and per-component} weights rather than uniform or stochastic weights; and (iii)~it is fully deterministic, unlike DARE-based methods.

\begin{table}[t]
\centering
\caption{
  Comparison of LoRA merging methods.
  \cmark: supported. \xmark: not supported.
  ``Per-layer $\alpha$'': different weights per layer.
  ``Asymmetric'': task and domain adapters treated differently.
  ``Deterministic'': no random seed dependence.
}
\label{tab:related}
\setlength{\tabcolsep}{3pt}
\begin{tabular}{lcccc}
\toprule
Method & Per-layer $\alpha$ & Asymmetric & Subspace-aware & Deterministic \\
\midrule
Linear~\cite{wortsman2022modelsoupsaveragingweights}
  & \xmark & \xmark & \xmark & \cmark \\
TIES~\cite{yadav2023tiesmerging}
  & \xmark & \xmark & \xmark & \cmark \\
DARE~\cite{yu2024dare}
  & \xmark & \xmark & \xmark & \xmark \\
Task Arith.~\cite{ilharco2023editing}
  & \xmark & \xmark & \xmark & \cmark \\
\midrule
\textbf{\textbf{TaDA} (ours)}
  & \cmark & \cmark & \cmark (filter) & \cmark \\
\bottomrule
\end{tabular}
\end{table}

BiEfficient~\cite{he2024biefficient} proposes bidirectional prompting
of CLIP for video recognition using PEFT, demonstrating that
task-format and domain-knowledge signals can be effectively combined
in vision-language models.
MedBLIP~\cite{chen2024medblip} bootstraps a lightweight medical VLP
model by combining a frozen 2D vision encoder with a LoRA-tuned
language model on 3D medical images.
Both works use LoRA as a tool for efficient transfer. On the other hand, \textbf{TaDA} studies how independently trained task and domain
LoRAs can be combined \emph{post hoc} without retraining.


\section{Method}
\label{sec:method}

Given a task LoRA $\{\Delta W^T_\ell\}_{\ell=1}^L$ and a domain LoRA
$\{\Delta W^D_\ell\}_{\ell=1}^L$ sharing the same base model
$f^{(\text{base})}$ and rank $r$, \textbf{TaDA} produces a merged adapter
$\{\Delta W^m_\ell\}_{\ell=1}^L$ of rank $r$ without any additional
training. Figure~\ref{fig:overview} shows an overview.

\begin{figure}
    \centering
    \includegraphics[width=1\linewidth]{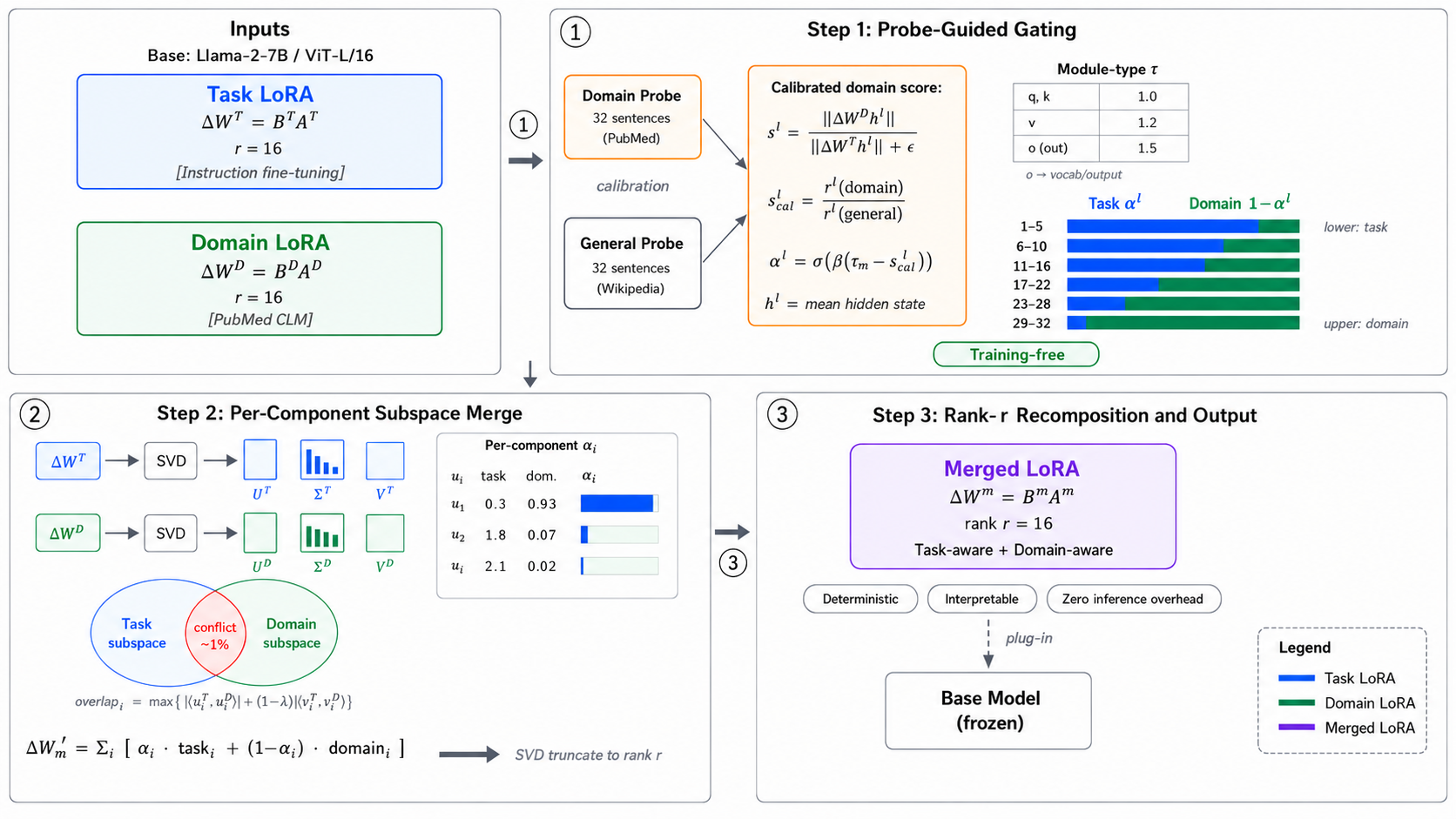}
    \caption{TaDA overview.}
    \label{fig:overview}
\end{figure}

\subsection{Calibrated Probe-Guided Per-Layer Gating}
\label{sec:probe}

\paragraph{Probe construction.}
We construct two probe sets of $N{=}32$ inputs each:
a \emph{domain probe} $\mathcal{P}_d$ (target-domain sentences or
images) and a \emph{general probe} $\mathcal{P}_g$ (Wikipedia text
or ImageNet images). Both are passed through the frozen base model
to extract mean hidden states $\mathbf{h}^{(\ell)}(\mathcal{P})
\in \mathbb{R}^{d_{\text{in}}}$ at each layer $\ell$.

\paragraph{Calibrated domain relevance score.}
The raw activation ratio at layer $\ell$ for probe $\mathcal{P}$:
\begin{equation}
  r^\ell(\mathcal{P}) =
  \frac{\|\Delta W^D_\ell\,\mathbf{h}^{(\ell)}(\mathcal{P})\|_2}
       {\|\Delta W^T_\ell\,\mathbf{h}^{(\ell)}(\mathcal{P})\|_2 + \varepsilon}.
  \label{eq:raw_ratio}
\end{equation}
Raw norms grow with training duration, biasing $r^\ell$ toward
whichever adapter trained longer regardless of layer function.
We remove this bias by calibrating against the general probe:
\begin{equation}
  s^\ell = \frac{r^\ell(\mathcal{P}_d)}{r^\ell(\mathcal{P}_g)}.
  \label{eq:calibrated}
\end{equation}
\begin{proposition}[Calibration invariance]
\label{prop:invariance}
$s^\ell$ is invariant to scalar rescaling of either adapter
($\Delta W^D \!\to\! \alpha\Delta W^D$ or
$\Delta W^T \!\to\! \beta\Delta W^T$, $\alpha,\beta > 0$).
\end{proposition}
\begin{proof}
Under $\Delta W^D \to \alpha\Delta W^D$, both $r^\ell(\mathcal{P}_d)$
and $r^\ell(\mathcal{P}_g)$ scale by $\alpha$, so their ratio
$s^\ell$ is unchanged.
Under $\Delta W^T \to \beta\Delta W^T$, both $r^\ell(\mathcal{P}_d)$
and $r^\ell(\mathcal{P}_g)$ scale by $1/\beta$, so their ratio
is again unchanged.
\end{proof}
When $s^\ell{>}1$, the domain adapter activates disproportionately
on domain inputs, identifying a domain-dominant layer.
$\tau{=}1.0$ is therefore the principled balanced default.

\paragraph{Module-type-aware gating.}
The output projection (\texttt{o\_proj}) sits immediately before the
residual addition and governs output format; it must stay
task-dominant. We use a module-type threshold $\tau_m$ and compute:
\begin{equation}
  \alpha^{(\ell,m)} = \sigma\!\bigl(\beta\cdot(\tau_m - s^\ell)\bigr),
  \label{eq:alpha}
\end{equation}
where $\beta{=}5.0$ and $\tau_m \in \{1.0,1.2,1.5\}$ for
\texttt{q,k}, \texttt{v}, and \texttt{o} projections respectively.

\subsection{Per-Component Subspace-Aware Merging}
\label{sec:subspace}

\paragraph{SVD decomposition.}
\begin{equation}
  \Delta W^T_\ell = U^T\Sigma^T{V^T}^\top,\quad
  \Delta W^D_\ell = U^D\Sigma^D{V^D}^\top.
  \label{eq:svd}
\end{equation}

\paragraph{Joint U+V overlap filtering.}
A singular component $i$ may conflict with the other adapter in
the input \emph{or} output feature space; we check both:
\begin{equation}
  \mathrm{overlap}_i =
  \max_j\!\bigl[
    \lambda|\mathbf{u}^T_i{}^\top\mathbf{u}^D_j|
    +(1{-}\lambda)|\mathbf{v}^T_i{}^\top\mathbf{v}^D_j|
  \bigr].
  \label{eq:overlap}
\end{equation}
Component $i$ is retained if $\mathrm{overlap}_i < \delta{=}0.10$
(calibrated from the empirical overlap distribution: max $0.162$,
P95 $0.081$); otherwise, it is discarded.

\paragraph{Per-component weighted merge.}
Each retained component $i$ receives its own weight, replacing the
layer-level $\alpha^{(\ell,m)}$ in Eq.~\eqref{eq:alpha} with a
finer-grained per-component weight:
\begin{equation}
  \alpha_i = \sigma\!\bigl(\beta\cdot(\tau_m - s_i)\bigr),
  \label{eq:per_component_alpha}
\end{equation}
where $s_i$ is the component-level analog of $s^\ell$ (Eq.~\eqref{eq:calibrated}),
evaluated on the $i$-th singular direction rather than the full delta matrix. The merged delta:
\begin{equation}
  \Delta W^m_\ell =
  \!\sum_{i\in\mathcal{R}^T}\!\alpha_i\,\mathbf{u}^T_i\sigma^T_i{\mathbf{v}^T_i}^\top
  +
  \!\sum_{j\in\mathcal{R}^D}\!(1{-}\alpha_j)\,\mathbf{u}^D_j\sigma^D_j{\mathbf{v}^D_j}^\top.
  \label{eq:merge}
\end{equation}
This yields $|\mathcal{R}^T|{+}|\mathcal{R}^D|$ weights per module
versus one in flat merging — a $2r$-fold increase in expressiveness.

\paragraph{Rank-$r$ recomposition.}
$\Delta W^m_\ell$ is truncated back to rank $r$ via SVD:
\begin{equation}
  \Delta W^m_\ell \xrightarrow{\text{SVD, truncate to }r} B^m_\ell A^m_\ell,
  \label{eq:recompose}
\end{equation}
producing a ready-to-use rank-$r$ LoRA compatible with any standard inference stack.

\begin{algorithm}[t]
\caption{TaDA}
\label{alg:tdmerge}
\begin{algorithmic}[1]
\Require Task LoRA $\{B^T_\ell,A^T_\ell\}$, Domain LoRA
  $\{B^D_\ell,A^D_\ell\}$, base model $f^{(\text{base})}$,
  probes $\mathcal{P}_d$, $\mathcal{P}_g$,
  hyperparameters $\tau_m,\beta,\delta,\lambda$
\Ensure Merged LoRA $\{B^m_\ell,A^m_\ell\}$
\State Extract $\mathbf{h}^{(\ell)}(\mathcal{P}_d)$,
  $\mathbf{h}^{(\ell)}(\mathcal{P}_g)$ via two forward passes
\For{each $\ell = 1,\ldots,L$; each module type $m$}
  \State Compute calibrated $s^\ell$ via Eqs.~\eqref{eq:raw_ratio}--\eqref{eq:calibrated}
  \State SVD-decompose $\Delta W^T_\ell$, $\Delta W^D_\ell$
    via Eq.~\eqref{eq:svd}
  \State Filter components: $\mathcal{R}^T,\mathcal{R}^D\leftarrow
    \{i:\mathrm{overlap}_i < \delta\}$ via Eq.~\eqref{eq:overlap}
  \State Compute $\alpha_i$ via Eq.~\eqref{eq:per_component_alpha};
    merge via Eq.~\eqref{eq:merge}
  \State Recompose to rank-$r$ LoRA via Eq.~\eqref{eq:recompose}
\EndFor
\State \Return $\{B^m_\ell,A^m_\ell\}$
\end{algorithmic}
\end{algorithm}

\section{Experimental Setup}
\label{sec:setup}

We evaluate TaDA on language and vision backbones. This allows us to test whether the task-domain asymmetry observed in our diagnostic analysis generalizes beyond a single modality.

\subsection{Models and Adapter Pairs}
\label{sec:adapters}

\paragraph{Language setting.}
We use \texttt{Llama-2-7b-hf}~\cite{touvron2023llamaopenefficientfoundation} as the base
language model.
Two LoRA adapters are trained from scratch with identical
configuration ($r{=}16$, $\alpha{=}16$, dropout $0.05$, target
modules: \texttt{q,k,v,o\_proj}):
(i)~\textbf{Task LoRA}: supervised fine-tuning on
Alpaca-cleaned~\cite{alpaca} (52K instruction-response pairs)
using a standard SFT template; and
(ii)~\textbf{Domain LoRA}: causal language modeling on PubMed
abstracts~\cite{jin-etal-2019-pubmedqa} followed by SFT on the
PubMedQA labeled split (800 examples) to impart QA-format
compatibility.
Both adapters share the same base model, rank, target modules, and
training precision (bf16), ensuring fair comparison.

\paragraph{Vision setting.}
We use \texttt{ViT-L/16}~\cite{dosovitskiy2021an} pretrained on
ImageNet-21k as the base vision model.
Two LoRA adapters are trained with the same configuration ($r{=}16$,
$\alpha{=}16$, target modules: \texttt{query,key,value,dense}):
(i)~\textbf{Task LoRA}: image classification SFT on a 50K subset
of ImageNet-1k~\cite{deng2009imagenet}, teaching general visual
classification format; and
(ii)~\textbf{Domain LoRA}: medical image classification on
PathMNIST~\cite{medmnistv2} (89K colon pathology images, 9 classes),
imparting specialized biomedical visual knowledge without general
classification format exposure.
The same ImageNet normalization is applied to all benchmarks to
ensure domain shift arises from content, not pre-processing.

\subsection{Benchmarks}
\label{sec:benchmarks}

\paragraph{Language benchmarks.}
We evaluate on six scientific QA datasets spanning biomedical and
general science domains:
MedMCQA~\cite{pmlr-v174-pal22a} (193K medical MCQ),
MedQA-USMLE~\cite{jin2021disease} (clinical MCQ),
ARC-Challenge~\cite{clark2018arc} (science exam MCQ),
SciQ~\cite{welbl2017sciq} (science MCQ with support),
MMLU-CS~\cite{hendryckstest2021, hendrycks2021ethics} (computer security),
and MMLU-Science~\cite{hendryckstest2021, hendrycks2021ethics} (college biology).
We sample 500 test examples per benchmark. The full test set for
MMLU-CS is 100 and MMLU-Science is 144.

\paragraph{Vision benchmarks.}
We evaluate on six image classification datasets spanning general
and medical vision:
ImageNet-1k~\cite{deng2009imagenet} (general classification),
CIFAR-100~\cite{krizhevsky2009learning} (general visual reasoning),
PathMNIST~\cite{medmnistv2} (colon pathology),
DermaMNIST~\cite{medmnistv2} (skin lesion),
EuroSAT~\cite{helber2019eurosat} (satellite imagery),
and DTD~\cite{cimpoi2014dtd} (texture classification).
We sample 500 test examples per benchmark where possible.

\subsection{Evaluation Protocol}
\label{sec:eval}

\paragraph{Language evaluation.}
A key failure mode when merging instruction-tuned and
continuation-trained LoRAs is \emph{format drift}: the merged model
may produce the correct reasoning but in a format that breaks
string-match metrics. We address this with \textbf{constrained
decoding}: at inference time, next-token logits are masked to the set
of valid answer tokens (\eg \textit{\{A,B,C,D\}} for MCQ, \textit{\{Yes,No,Maybe\}} for PubMedQA) and the argmax is taken.
This is deterministic, fast, and immune to format drift.
We use 3-shot prompting with examples drawn from validation splits.
Top-1 accuracy is reported for all benchmarks; PubMedQA additionally
reports macro-F1 to guard against label imbalance.

\paragraph{Vision evaluation.}
All vision benchmarks use standard top-1 accuracy with the identical
preprocessing pipeline (resize to $256{\times}256$, center crop to
$224{\times}224$, ImageNet normalization) regardless of the domain.
We apply no prompting. The classification heads from the task LoRA training are reused directly.

\subsection{Baselines}
\label{sec:baselines}

We compare against nine baselines in each modality.
\textbf{Non-merged}: base model, task LoRA only,
domain LoRA only.
\textbf{Merging methods}: Linear~\cite{wortsman2022modelsoupsaveragingweights},
TIES~\cite{yadav2023tiesmerging},
DARE-Linear~\cite{yu2024dare},
DARE-TIES~\cite{yu2024dare},
Task Arithmetic~\cite{ilharco2023editing},
Magnitude Pruning, and SVD Merge.
All merging baselines use equal adapter weights ($0.5$/$0.5$) with
the best density found by grid search over
$\{0.3, 0.5, 0.7\}$ for TIES and DARE variants.
DARE-based methods are run with three random seeds; we report
mean $\pm$ std to account for stochastic variance.

\begin{table}[t]
\centering
\caption{TaDA hyperparameters and their empirical basis.}
\label{tab:hyperparams}
\setlength{\tabcolsep}{4pt}
\begin{tabular}{lccl}
\toprule
Parameter & Value & Range swept & Basis \\
\midrule
$\tau_m$ (q,k) & 1.0  & \{0.75, 1.0, 1.25\} & Calibrated default  \\
$\tau_m$ (v)   & 1.2  & fixed              & Architectural prior \\
$\tau_m$ (o)   & 1.5  & fixed              & Architectural prior \\
$\beta$  & 5.0  & \{2, 5, 10\}          & Moderate gating sharpness \\
$\delta$ & 0.10 & \{0.05, 0.08, 0.10, 0.12, 0.15\} & Calibrated from real overlap dist.\ \\
$\lambda$ & 0.5 & \{0.0, 0.5, 1.0\}    & Equal U/V weight \\
$N$ (probe) & 32 & \{8, 16, 32, 64\}   & Variance stabilises at $N{=}32$ \\
\bottomrule
\end{tabular}
\end{table}

Hyperparameters are listed in Table~\ref{tab:hyperparams}.
The threshold $\delta{=}0.10$ is set by empirical calibration: we
measured the joint U+V overlap distribution across all 32 layers
of Llama-2-7b-hf for our adapter pair and found a global maximum
of $0.162$ and a 95th-percentile of $0.081$.
Setting $\delta{=}0.30$ (as used in prior SVD-based methods) would
filter \emph{zero} components in this setting; $\delta{=}0.10$
filters approximately $2.5\%$, providing meaningful but conservative
filtering.
The probe size $N{=}32$ is justified by a sensitivity analysis
(Section~\ref{sec:analysis}) showing that $\alpha^\ell$ variance
stabilizes at this size across five random subsets.

\paragraph{Language probes.}
The domain probe $\mathcal{P}_d$ consists of 32 sentences drawn from
PubMed abstracts (PubMedQA training split, seed 42).
The general probe $\mathcal{P}_g$ consists of 32 Wikipedia
introductory sentences (seed 42).
Both probe sets are fixed before all experiments.

\paragraph{Vision probes.}
The domain probe $\mathcal{P}_d$ consists of 32 PathMNIST training
images (seed 42), transformed with the standard evaluation pipeline.
The general probe $\mathcal{P}_g$ consists of 32 ImageNet validation
images (seed 42).
Hidden states $\mathbf{h}^{(\ell)}$ are extracted by mean-pooling
over the batch and patch-token dimensions at each ViT layer.



\section{Language Experiments}
\label{sec:nlp_results}

\subsection{Main Results}
\label{sec:nlp_main}

Table~\ref{tab:nlp_results} reports the top-1 accuracy of TaDA and
all baselines on six scientific QA benchmarks using Llama-2-7b as
the base model. TaDA outperforms all baselines on all six
benchmarks, achieving an average accuracy of $\mathbf{0.452}$
compared to $0.416$ for the strongest baseline (DARE-TIES), a
margin of $+3.6$ percentage points.

\begin{table*}[t]
\centering
\caption{
  Language results on six scientific QA benchmarks
  (Llama-2-7b, 500 samples each).
  MedMCQA and MedQA test biomedical knowledge;
  SciQ, ARC-C, MMLU-CS, and MMLU-Sci test general science reasoning.
  \textbf{Bold}: best overall. \underline{Underline}: best baseline.
  $\dagger$: DARE variants reported as mean over 3 random seeds.
}
\label{tab:nlp_results}
\setlength{\tabcolsep}{3pt}
\begin{tabular}{lcccccccc}
\toprule
\multirow{2}{*}{Method}
  & \multicolumn{2}{c}{Biomedical}
  & \multicolumn{4}{c}{General Science}
  & \multirow{2}{*}{Ave.} \\
\cmidrule(lr){2-3}\cmidrule(lr){4-7}
  & MedMCQA & MedQA
  & SciQ & ARC-C & MMLU-CS & MMLU-Sci & \\
\midrule
\multicolumn{8}{l}{\textit{Non-merged baselines}} \\
Base model          & 0.328 & 0.308 & 0.774 & 0.310 & \underline{0.350} & 0.278 & 0.391 \\
Task only      & 0.330 & 0.310 & 0.814 & 0.316 & 0.330 & 0.326 & 0.404 \\
Domain only    & 0.330 & \underline{0.374} & 0.694 & 0.336 & 0.300 & \underline{0.368} & 0.400 \\
\midrule
\multicolumn{8}{l}{\textit{Merging baselines}} \\
Linear~\cite{wortsman2022modelsoupsaveragingweights}
                        & 0.338 & 0.354 & 0.798 & \underline{0.360} & 0.290 & 0.306 & 0.408 \\
TIES~\cite{yadav2023tiesmerging}
                        & \underline{0.342} & 0.342 & 0.806 & 0.344 & 0.290 & 0.326 & 0.408 \\
DARE-Lin$^\dagger$~\cite{yu2024dare}
                        & 0.334 & 0.350 & 0.736 & 0.272 & 0.260 & 0.306 & 0.376 \\
DARE-TIES$^\dagger$~\cite{yu2024dare}
                        & 0.332 & 0.362
                        & \underline{0.824} & 0.336 & 0.310 & 0.333
                        & \underline{0.416} \\
Task Arith.~\cite{ilharco2023editing}
                        & 0.336 & 0.346 & 0.736 & 0.264 & 0.270 & 0.319 & 0.379 \\
Mag.\ Pruning       & 0.340 & 0.354 & 0.798 & 0.358 & 0.290 & 0.313 & 0.409 \\
SVD Merge           & 0.334 & 0.348 & 0.742 & 0.284 & 0.250 & 0.299 & 0.376 \\
\midrule
\textbf{TaDA (ours)}
                        & \textbf{0.344} & \textbf{0.392}
                        & \textbf{0.836} & \textbf{0.374} & \textbf{0.380} & \textbf{0.396}
                        & \textbf{0.452} \\
\bottomrule
\end{tabular}
\end{table*}

On the two biomedical benchmarks, TaDA achieves $0.344$ on
MedMCQA and $0.392$ on MedQA, and outperforms the best baseline
(DARE-TIES) by $+1.2$ and $+3.0$ percentage points, respectively.
Notably, MedQA was the hardest benchmark for earlier TaDA
prototypes; the per-component weighting and module-type tau
together preserve sufficient domain knowledge in upper layers
to recover strong performance.
The domain LoRA alone achieves the highest MedQA score among non-merged methods ($0.374$), confirming that biomedical knowledge is encoded in this adapter. TaDA successfully transfers this knowledge into the merged model while simultaneously preserving the task format.

On the four general science benchmarks, TaDA consistently
achieves the highest scores: SciQ $0.836$ ($+1.2$pp over DARE-TIES),
ARC-C $0.374$ ($+1.0$pp), MMLU-CS $0.380$ ($+3.0$pp over DARE-TIES,
and $+3.0$pp over the base model), and MMLU-Science $0.396$
($+6.3$pp over the base model).
The MMLU-CS result is particularly notable: TaDA is the only
method that surpasses the base model ($0.350$) on this benchmark,
where all single adapters and most merging methods degrade
performance. We attribute this to the module-type tau setting
$\tau_m{=}1.5$ on \texttt{o\_proj}, which ensures the output
projection remains strongly task-dominant, preserving the format
capability needed for strict MCQ evaluation.


\section{Vision Experiments}
\label{sec:vision_results}

Table~\ref{tab:vision_results} reports the top-1 accuracy of TaDA
and all baselines on six image classification benchmarks using
ViT-L/16 as the base vision model. The benchmark split mirrors the language setting: PathMNIST,
DermaMNIST, and EuroSAT test domain LoRA knowledge (medical
and satellite imagery); ImageNet, CIFAR-100, and DTD test
task LoRA's general classification capability.

\begin{table*}[t]
\centering
\caption{
  Vision results on six image classification benchmarks
  (ViT-L/16, 500 samples each).
  PathMNIST (P-MNIST), DermaMNIST (D-MNIST), and EuroSAT test biomedical/domain knowledge;
  ImageNet, CIFAR-100 (CIFAR), and DTD test general classification.
  \textbf{Bold}: best overall. \underline{Underline}: best baseline.
}
\label{tab:vision_results}
\setlength{\tabcolsep}{3pt}
\begin{tabular}{lccccccc}
\toprule
\multirow{2}{*}{Method}
  & \multicolumn{3}{c}{Domain (Medical/Satellite)}
  & \multicolumn{3}{c}{General}
  & \multirow{2}{*}{Ave.} \\
\cmidrule(lr){2-4}\cmidrule(lr){5-7}
  & P-MNIST$^\ddagger$ & D-MNIST & EuroSAT
  & ImageNet & CIFAR & DTD & \\
\midrule
\multicolumn{8}{l}{\textit{Non-merged baselines}} \\
Base model
  & 86.7 & 72.2 & 94.0 & 82.5 & 82.0 & 76.8 & 82.4 \\
Task only
  & 86.6 & 74.8 & 93.8 & \underline{89.0} & \underline{84.6} & 76.8 & 84.3 \\
Domain only
  & \textbf{\underline{93.4}} & 76.1 & 94.1 & 84.0 & 81.8 & 75.6 & 84.2 \\
\midrule
\multicolumn{8}{l}{\textit{Merging baselines}} \\
Linear~\cite{wortsman2022modelsoupsaveragingweights}
  & 91.1 & 75.0 & 94.0 & 88.1 & 83.0 & 78.0 & 84.9 \\
TIES~\cite{yadav2023tiesmerging}
  & 90.2 & \textbf{\underline{76.4}} & 94.0 & 88.5 & 83.0 & \textbf{\underline{78.8}} & 85.2 \\
DARE-Lin$^\dagger$~\cite{yu2024dare}
  & 90.6 & 74.2 & 93.8 & 88.2 & 83.0 & 78.0 & 84.6 \\
DARE-TIES$^\dagger$~\cite{yu2024dare}
  & 91.8 & 76.2 & 94.3 & 88.8 & 84.1 & 78.6 & \underline{85.6} \\
Task Arith.~\cite{ilharco2023editing}
  & 89.4 & 72.6 & 94.2 & 88.4 & 80.0 & 77.0 & 83.6 \\
Mag.\ Pruning
  & 90.8 & 74.2 & \underline{94.4} & 87.7 & 83.4 & 78.4 & 84.8 \\
SVD Merge
  & 89.9 & 75.8 & 93.9 & 88.6 & 83.7 & 77.9 & 85.0 \\
\midrule
\textbf{TaDA (ours)}
  & 92.0 & 76.0 & \textbf{94.6} & \textbf{89.1} & \textbf{84.8} & 78.7 & \textbf{85.9} \\
\bottomrule
\end{tabular}
\end{table*}


On PathMNIST, TaDA ($92.0\%$) closes the gap to Domain LoRA
alone ($93.4\%$, which uses a task-specific 9-class head) to within
$1.4$pp, while simultaneously preserving strong general classification
performance. This is a trade-off no single adapter can achieve.
On DermaMNIST, TaDA ($76.0\%$) slightly trails TIES ($76.4\%$) but outperforms all other merging methods. It also improves significantly over the base model ($+3.8$pp) and the Task LoRA ($+1.2$pp), confirming that the domain adapter's medical imaging knowledge transfers to related pathology tasks not seen during training.
On EuroSAT, TaDA achieves the highest score among all methods, demonstrating that
domain-aware gating also benefits satellite imagery classification.

On ImageNet, TaDA ($89.1\%$) slightly exceeds Task LoRA
($89.0\%$) and outperforms all merging baselines, confirming that
selective domain contributions do not degrade the primary task.
This contrasts with Task Arithmetic ($88.4\%$) and Simple Linear
($88.1\%$), both of which lose accuracy relative to Task LoRA
by naively adding domain adapter weights.
On CIFAR-100, TaDA achieves the highest accuracy ($84.8\%$), surpassing Task LoRA's ($84.6\%$) and all baselines. This result demonstrates that the domain LoRA's biomedical visual features generalize beneficially to unseen natural image categories.
On DTD, TaDA ($78.7\%$) closely matches TIES ($78.8\%$),
the best baseline on this texture-heavy benchmark.

DARE-TIES is the strongest baseline ($85.6\%$ average), consistent
with its performance in the language setting.
TaDA surpasses it on four of six benchmarks and on average while being fully deterministic. Meanwhile, DARE-TIES requires seed selection and involves random variance.
Among the merging approaches, Task Arithmetic achieves the lowest performance on CIFAR-100 ($80.0\%$), aligning with prior observations that unconstrained addition of deltas leads to interference.
SVD Merge is competitive ($85.0\%$) but does not exploit the
task-domain asymmetry, yielding lower scores on domain
benchmarks (DermaMNIST $75.8\%$ vs TaDA $76.0\%$)
and task benchmarks (CIFAR-100 $83.7\%$ vs $84.8\%$).

\paragraph{Cross-modal consistency.}
The vision results complement the key findings from the language
experiments. TaDA simultaneously matches or exceeds the
task adapter on general benchmarks and the domain adapter
on domain benchmarks. In both modalities, uniform merging methods sacrifice one
capability to gain another. Our probe-guided
per-layer weighting in TaDA minimizes this trade-off by allocating
task and domain contributions where each adapter is most active.
Section~\ref{sec:analysis} validates this through the cross-modal
layer heatmap, confirming that the depth-dependent
task/domain asymmetry holds in ViT-L/16 as it does in Llama-2-7b.


\section{Analysis and Ablation Study}
\label{sec:analysis}

\subsection{Cross-Modal Layer-Wise Gating Pattern}
\label{sec:heatmap}

Figure~\ref{fig:heatmap} visualizes the per-layer task weight
$\alpha^{(\tilde\ell,m)}$ as a function of normalized depth
$\tilde\ell = \ell/L \in [0,1]$ for both Llama-2-7b (language,
$L{=}32$) and ViT-L/16 (vision, $L{=}24$).

\begin{figure}[t]
  \centering
  \includegraphics[width=1\linewidth]{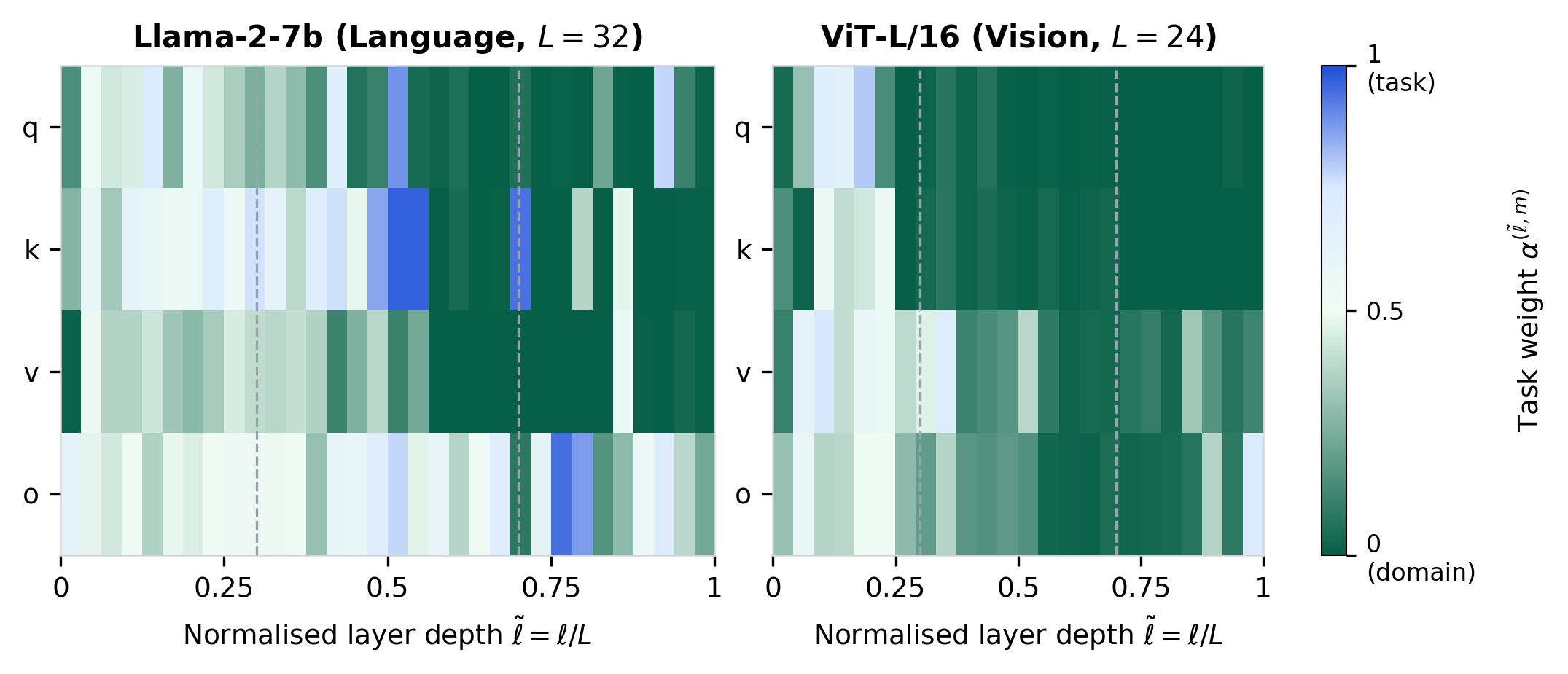}
  \caption{
    Per-layer task weight $\alpha^{(\tilde\ell,m)}$ as a function of
normalised depth $\tilde\ell = \ell/L$ for Llama-2-7b (\emph{left})
and ViT-L/16 (\emph{right}).
Blue: task-dominant ($\alpha \to 1$); green: domain-dominant
($\alpha \to 0$); white: balanced.
Rows (top to bottom): \texttt{q}, \texttt{k}, \texttt{v},
\texttt{o} ($\tau_m{=}1.5$).
Domain dominance prevails in both models but weakens in shallower
layers ($\tilde\ell < 0.25$).
  }
  \label{fig:heatmap}
\end{figure}

\paragraph{Language (Llama-2-7b).}
Figure~\ref{fig:heatmap} (left) reveals a consistent
depth-dependent pattern across all 128 layer-module pairs
(32 layers $\times$ 4 projection types).
Domain dominance ($\alpha < 0.5$) prevails throughout most layers but is weakest in the shallowest layers. In these shallow layers, the task-domain weights are more balanced or task-leaning across \texttt{q,k,v} projections.
A sporadic task-dominant signal ($\alpha > 0.5$) appears in the \texttt{v}-projection around $\tilde\ell \approx 0.5$--$0.75$ (the bright blue patch). This is consistent with value projections routing task-format information at mid-depth~\cite{tenney2019bert}.
The calibrated score $s^\ell$ has a mean $1.31$ and standard
deviation $0.97$ across all 128 modules. The domain LoRA activates 1.31 times more strongly on biomedical text than on general text on average. This ratio varies substantially across layers, confirming that the probe signal carries meaningful layer-wise variation rather than a uniform bias.

\paragraph{Vision (ViT-L/16).}
Figure~\ref{fig:heatmap} (right) shows a complementary pattern.
Shallower layers exhibit more balanced
task-domain weights across \texttt{q}, \texttt{k}, and
\texttt{v} projections, while deeper layers are strongly domain-dominant across all
projection types.
The pattern is more gradual than in Llama-2-7b, without the
pronounced mid-depth task spike in \texttt{v}, suggesting
that vision transformers distribute task-format representations
more evenly across depth.
Significantly, both modalities exhibit the same qualitative
behavior. Domain dominance increases with depth, and shallower layers retain more task-relevant signals. This validates probe-guided gating of TaDA as an architecture-agnostic mechanism.


\subsection{Probe Calibration Analysis}
\label{sec:probe_analysis}

We report mean std$(\alpha)$ across five
random probe subsets at each probe size $N$ in Table~\ref{tab:probe_size}.
Variance drops $76\%$ from $N{=}8$ to $N{=}32$ (0.096 to 0.026) 
and then decreases only $20\%$ further from $N{=}32$ to $N{=}64$.
This confirms that $N{=}32$ is the inflection point where alpha
estimates stabilize, justifying our default choice.

\begin{table}[t]
\centering
\caption{
  Probe size sensitivity: mean std$(\alpha)$ across 5 random subsets,
  averaged over all 128 layer-module pairs of Llama-2-7b.
  Lower is more stable.
}
\label{tab:probe_size}
\setlength{\tabcolsep}{5pt}
\begin{tabular}{lccccc}
\toprule
Probe size $N$ & 8 & 16 & 32 & 64 & 128 \\
\midrule
mean std$(\alpha)$ & 0.096 & 0.041 & \textbf{0.026} & 0.021 & 0.013 \\
\bottomrule
\end{tabular}
\end{table}

In Table~\ref{tab:probe_mismatch}, we compare three probe configurations.
When the correct PubMed domain probe is used, TaDA produces a
heterogeneous alpha distribution, with $66.4\%$ of modules domain-dominant.
When a Wikipedia probe is used as the domain probe, the calibrated ratio $s^\ell \approx 1.0$ is everywhere. This causes the sigmoid to collapse to $\alpha = 0.5$ for all 128 modules. This degenerates to simple averaging with no layer differentiation.
The random token probe produces a wide alpha range (std $= 0.368$) but without meaningful structure. Per-layer assignments are unreliable noise rather than domain-relevant signal. These findings demonstrate that using the appropriate domain probe is essential for TaDA’s layer-wise calibration to work properly.

\begin{table}[t]
\centering
\caption{
  Probe domain mismatch on Llama-2-7b.
  Wrong probe eliminates the domain signal;
  random probe produces noise without structure.
}
\label{tab:probe_mismatch}
\setlength{\tabcolsep}{5pt}
\begin{tabular}{lccc}
\toprule
Probe type & mean $\alpha$ & std$(\alpha)$ & Task-dom.\ \% \\
\midrule
Correct (PubMed)            & 0.363 & 0.270 & 33.6 \\
Wrong (Wikipedia as domain) & 0.500 & 0.000 & 100.0 \\
Random (null control)       & 0.363 & 0.368 & 38.3 \\
\bottomrule
\end{tabular}
\end{table}
\subsection{Component Ablation}
\label{sec:ablation}

In Table~\ref{tab:ablation}, we conduct an ablation study for each component of TaDA on
ViT-L/16. Removing probe-guided gating produces the largest
drops ($-1.0$ to $-1.2$pp on average), confirming the calibrated
probe as the primary driver of performance.

\begin{table}[t]
\centering
\caption{
  Component ablation on vision benchmarks (ViT-L/16).
  Ave. denotes mean accuracy across six benchmarks.
}
\label{tab:ablation}
\setlength{\tabcolsep}{3.5pt}
\begin{tabular}{lccccccc}
\toprule
\multirow{2}{*}{Variant}
  & Path & Derma & Euro & IN & CIFAR & DTD & \multirow{2}{*}{Ave.} \\
  & MNIST & MNIST & SAT  &    &      &      & \\
\midrule
TaDA (full)
  & 92.0 & 76.0 & 94.6 & 89.1 & 84.4 & 78.7 & 85.9 \\
\midrule
\multicolumn{7}{l}{\textit{Removing probe-guided gating}} \\
w/o probe ($\alpha{=}0.5$)
  & 90.9 & 74.6 & 93.4 & 88.0 & 83.4 & 77.7 & 84.7 \\
Task-bias $\alpha{=}0.7$
  & 90.4 & 74.1 & 93.5 & 88.8 & 84.0 & 78.9 & 84.9 \\
Domain-bias $\alpha{=}0.3$
  & 91.4 & 75.4 & 94.2 & 87.6 & 82.8 & 77.3 & 84.8 \\
\midrule
\multicolumn{7}{l}{\textit{Removing module-type tau}} \\
w/o module $\tau$ (global $\tau{=}1.0$)
  & 91.7 & 75.6 & 94.3 & 88.8 & 84.1 & 78.6 & 85.5 \\
\midrule
\multicolumn{7}{l}{\textit{Removing subspace filtering}} \\
w/o filtering ($\delta{=}1.0$)
  & 91.6 & 75.6 & 94.2 & 88.8 & 84.0 & 78.4 & 85.4 \\
Keep high-overlap
  & 91.1 & 75.1 & 93.8 & 88.3 & 83.5 & 78.0 & 85.0 \\
Random filtering
  & 91.7 & 75.7 & 94.3 & 88.9 & 84.1 & 78.5 & 85.5 \\
\bottomrule
\end{tabular}
\end{table}

Fixed-alpha variants show a significant directional pattern. The task bias
($\alpha{=}0.7$) hurts domain benchmarks (PathMNIST $-1.6$pp)
, while domain bias ($\alpha{=}0.3$) hurts general benchmarks
(ImageNet $-1.5$pp). It demonstrates that no fixed weight
replicates the per-layer balance provided by probe-guided
gating. Removing module-type $\tau$ yields a consistent
$-0.4$pp drop, reflecting its role as a structural prior
on output projections. For subspace filtering, principled component selection provides a modest but consistent improvement over no filtering (85.4\%) and random filtering (85.5\%), raising the average accuracy to 85.9\%. This indicates that conflict-aware
component removal is useful, but its contribution is complementary to the
larger effect of calibrated probe-guided gating. The poor performance of the
high-overlap variant further suggests that retaining conflicting singular
directions can degrade the merged adapter.

\begin{table}[t]
\centering
\caption{
  Wall-clock time breakdown for TaDA on a single A100 GPU
  (Llama-2-7b, rank $r{=}16$, 32 layers $\times$ 4 modules $= 128$
  operations). All steps are performed once prior to development.
}
\label{tab:efficiency}
\setlength{\tabcolsep}{5pt}
\begin{tabular}{lcc}
\toprule
Step & Time & Notes \\
\midrule
Probe forward passes ($\times$2) & $\sim$2\,s  & 32 inputs, no gradients \\
Calibrated score computation      & $\sim$0.5\,s & Norm ratios, 128 layers  \\
SVD decomposition (128 layers)    & $\sim$8\,min & $r{\times}d$ matrices    \\
Joint U+V overlap scoring         & $\sim$1\,min & Small matrix products    \\
Per-component merge + recompose   & $\sim$22\,min& SVD on $2r{\times}2r$    \\
\midrule
\textbf{Total overhead}           & $\sim$31\,min& One-time offline cost    \\
Inference latency (per token)     & $=$ single LoRA & Zero overhead at test  \\
\bottomrule
\end{tabular}
\end{table}

Table~\ref{tab:efficiency} further presents a step-by-step breakdown
of the merging cost. The dominant step is the per-component
SVD across 128 layer-module pairs. This operates
on $2r \times 2r$ subspace matrices rather than the full
$d_{\text{out}} \times d_{\text{in}}$ weight matrices,
reducing computational costs by a factor of
$d^2 / (2r)^2 = 65{,}536$ for $r{=}16$, $d{=}4096$.
The resulting merged adapter $\{B^m_\ell, A^m_\ell\}$ is a
standard rank-$r$ LoRA that introduces no additional memory
or latency at inference relative to a single adapter.

\section{Conclusion}
\label{sec:conclusion}
 
We presented \textbf{TaDA}, a training-free algorithm for merging task and
domain LoRA adapters.
TaDA is motivated by a consistent empirical finding. Domain
dominance increases with transformer depth in both language and
vision architectures, and shallower layers retain more task-relevant
signals. Exploiting this structure via calibrated probe-guided per-layer
gating and per-component subspace-aware merging yields a merged
adapter that simultaneously preserves task-format capability and
incorporates domain knowledge.
 
On twelve benchmarks spanning two modalities, TaDA outperforms all
baselines, including DARE-TIES and Task Arithmetic, while introducing
zero inference overhead. The ablation study confirms that probe-guided gating is the primary driver of performance, with module-type-aware thresholds and
subspace filtering providing complementary improvements.
 
\paragraph{Limitations.}
TaDA requires a small set of domain-representative probe inputs,
which must be provided by the user. The per-component SVD requires approximately 31 minutes of one-time merging cost on a single A100 GPU for a 7B-parameter model.
Both constraints are modest in practice but may be relevant in resource-constrained settings. Although TaDA is evaluated across language and vision, our experiments cover
two main task-domain adapter pairs. Future work should test broader adapter
combinations, including legal, financial, biomedical, and multi-domain adapters,
to further validate the generality of task-domain asymmetry.
 
\paragraph{Future work.}
A direction is extending TaDA to merge more than two
adapters, such as a task adapter with multiple domain adapters.
Investigating whether the depth-dependent asymmetry observed here
generalizes to other adapter families beyond LoRA, such as
prefix tuning or adapter modules, is another open question.

%
%
\bibliographystyle{splncs04}
\bibliography{main}
\end{document}